\PassOptionsToPackage{dvipsnames,svgnames,x11names}{xcolor}
\documentclass{article}
\usepackage{arxiv}

\usepackage[utf8]{inputenc} 
\usepackage[T1]{fontenc}    
\usepackage{palatino}
\usepackage{hyperref}       
\usepackage{url}            
\usepackage{booktabs}       
\usepackage{amsfonts}       
\usepackage{nicefrac}       
\usepackage{microtype}      
\usepackage{xcolor}         
\usepackage{hyperref}
\usepackage{url}
\usepackage{eso-pic}
\usepackage{fancyhdr}
\usepackage{graphicx}
\usepackage{tikz}
\usepackage{wrapfig}
\usepackage{bbm}
\usepackage{amsmath}
\usepackage{amssymb}
\usepackage{mathtools}
\usepackage{amsthm}
\usepackage{enumitem}
\usepackage{xspace}
\usepackage{xparse}
\usepackage{natbib}
\usepackage{xcolor}
\usepackage{multirow}
\usepackage{setspace}
\usepackage{enumitem}
\usepackage{color, colortbl}
\usepackage{mathtools}
\usepackage{cleveref}
\theoremstyle{plain}
\newtheorem{theorem}{Theorem}[section]
\usepackage[linesnumbered,ruled,vlined]{algorithm2e}
\usepackage{listings}
\usepackage{color,graphicx}
\usepackage{siunitx}
\usepackage{soul}
\usepackage{graphics} 
\usepackage{bookmark}
\usepackage[dvipsnames, svgnames, x11names]{xcolor}
\usepackage{hyperref} 
\usepackage{textcomp}
\usepackage{multirow}
\usepackage{nth}
\usepackage{indentfirst}
\usepackage{siunitx}
\usepackage{changepage}
\usepackage{bm}
\usepackage{setspace}
\usepackage{multirow}
\usepackage{multicol}
\usepackage{colortbl}
\usepackage{booktabs}
\usepackage{amsmath}
\usepackage{xcolor}
\usepackage{amsthm}
\theoremstyle{remark}

\usepackage{graphicx}
\usepackage{subfigure}
\usepackage{wrapfig}
\usepackage{enumitem}
\usepackage[most]{tcolorbox}
\definecolor{PPOGray}{HTML}{6B7280}      
\definecolor{SPOOrange}{HTML}{D97706}    
\definecolor{GTRTeal}{HTML}{0F766E}      
\definecolor{darkmagenta}{rgb}{0.66, 0.3, 0.3}  
\usepackage{parskip}
\usepackage{titling}

\newcommand{\High}{\textbf{\textcolor{red!75!black}{High \ }}}
\newcommand{\Medium}{\textbf{\textcolor{orange!90!black}{Medium \ }}}
\newcommand{\Low}{\textbf{\textcolor{green!50!black}{Low \ }}}
\newcommand{\Pos}{\textbf{\textcolor{blue!75!black}{Pos. \ }}}
\newcommand{\Neg}{\textbf{\textcolor{orange!85!black}{Neg. \ }}}
\newtcolorbox{takeawaybox}{
  colback=blue!3,
  boxrule=0.8pt,
  arc=1.5mm,
  left=1.2mm,
  right=1.2mm,
  top=1mm,
  bottom=1mm,
  fonttitle=\bfseries,
  title=Take-Away
}

\hypersetup{colorlinks,linkcolor={blue},citecolor={magenta},urlcolor={blue}}

\usepackage{xurl}           
\NewDocumentCommand{\drsh}{ O{0.6em} O{0.5em} O{0.65pt} O{rounded corners=1.6pt} }{%
  \mathrel{%
  \hspace{0.3em}
    \tikz[baseline=-0.6em]{%
      \draw[->, line width=#3, #4] (0,0) -- (0,-#2) -- (#1,-#2);
    }%
  }%
}

\newcommand{\ourmethod}{\texttt{DC-GRPO}\xspace}
\newcommand{\mycolor}{cyan!10}

\newcommand{\QEightB}{Qwen3-8B\xspace}
\newcommand{\QThirtyB}{Qwen3-30B-A3B\xspace}

\title{Reformulate LLM Reinforcement Learning for Efficient Training under Black-box Discrepancy} 

\author{
\textbf{Jiashun Liu}$^{1}$\protect\footnotemark[1] \quad
\textbf{Runze Liu}$^{1}$\protect\footnotemark[1] \quad
\textbf{Xu Wan}$^{2}$\protect\footnotemark[1] \quad
\textbf{Jing Liang}$^{3}$\protect\footnotemark[1] \\
\textbf{Hongyao Tang}$^{3}$ \quad
\textbf{Ling Pan}$^{1}$\protect\footnotemark[2] \\
$^{1}$Hong Kong University of Science and Technology \\
$^{2}$Zhejiang University \quad
$^{3}$Tianjin University
}

\date{}
\begin{document}
\renewcommand{\thefootnote}{\fnsymbol{footnote}}
\setcounter{footnote}{0}

\maketitle

\footnotetext[1]{Equal contribution}
\footnotetext[2]{Corresponding author}

\renewcommand{\thefootnote}{\arabic{footnote}}
\setcounter{footnote}{0}

\begin{abstract}
Reinforcement Learning (RL) has emerged as a pivotal post-training paradigm, yet it frequently suffers from unpredictable sub-optimum performance or even training collapses. Recent findings attribute these failures to a hidden train-inference discrepancy (or mismatch), stemming from the disparate underlying engines and architecture. We find that the training policy can actively self-correct such a discrepancy when provided with an appropriate learning signal. Then, we further empirically identify a \emph{discrepancy tolerance region}: within this region, aggressively narrowing the discrepancy can suppress policy exploration and reduce learning efficiency, whereas outside this region, reducing excessive discrepancy improves optimization consistency and raises the achievable local performance ceiling. According to such findings, we formulate this problem as a \textbf{Discrepancy-Constrained Markov Decision Process} (\texttt{DCMDP}), where reward maximization is coupled with a constraint that aligns training-Inference behavior, achieving stable dual-objective optimization. To adaptively balance performance improvement and discrepancy control, we introduce a Lagrangian relaxation mechanism that dynamically adjusts the relative weight of the two objectives according to the current degree of discrepancy violation. This enables stable dual-objective optimization: the policy is allowed to explore freely within the tolerance region, while being guided back when the discrepancy exceeds the safe boundary. Empirically, \texttt{DCMDP} significantly improves the performance of 8B dense model (\texttt{Qwen-3-8b}) and 30B Mixture-of-Expert model (\texttt{Qwen-3-30bA3b}), and enables a heterogeneous training paradigm, where LLMs can be optimized in high-fidelity training setup while being explicitly aligned for low-cost, resource-constrained inference deployment.
\end{abstract}
\begin{figure}[h]
\centering
\vspace{-0.6cm}
    \includegraphics[width=\linewidth]{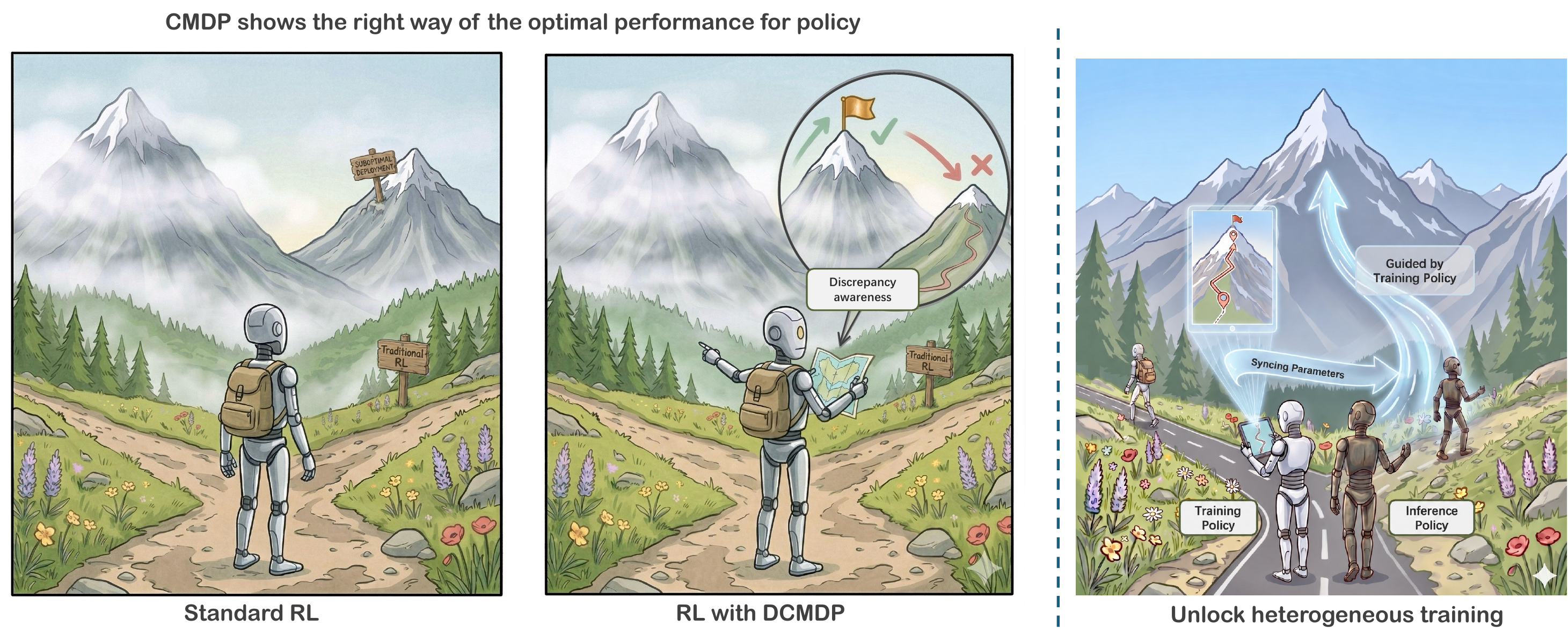}
    \caption{(Left): CMDP guides the training policy to find the optimal solution along the inference policy preferred solution space by injecting real-time discrepancy awareness. (Right): DCMDP enables heterogeneous training paradigm by treating infrastructure and architecture as black boxes.}
    \label{fig:1}
    \vspace{-0.4cm}
\end{figure}

\section{Introduction}
Recent advancements have established Reinforcement Learning (RL) as a pivotal post-training paradigm \citep{DAPO,DeepSeek-R1,Open-Reasoner-Zero}, significantly enhancing the deep reasoning capabilities of Large Language Models (LLMs) and continuously pushing the boundaries of LLM applications in complex real-world scenarios \citep{Qwen2.5-Coder}. However, a persistent and costly challenge plagues the community: RL post-training often suffers abrupt, unpredictable training collapses \citep{Qi2025DefeatingTT,zheng2025stabilizing,team2025every}, leading to substantial waste of computational resources and time. Recent investigations have identified the underlying cause of these failures: a pervasive train-inference mismatch arising from conflicting industrial demands \citep{li2026qurl}. Specifically, to balance the extreme throughput necessary for generation with the stringent numerical precision required for parameter updates, practitioners are often forced to deploy disparate underlying engines and precision formats for the exact same model weights \citep{zhang2026beyond,dong2026probing}. This discrepancy, residing entirely beneath the algorithmic layer, induces a deviation between the action probability distributions of the inference policy and the training policy.

A straightforward approach to mitigate this persistent issue is that merely aligning the precision of training and inference, by uniformly downgrading to FP16, can alleviate training collapses \citep{Qi2025DefeatingTT}. However, this strategy drastically shrinks both the maximum and minimum representable values and exacerbates the risk of gradient explosion due to numerical overflow \citep{zheng2025stabilizing}. Furthermore, precision is merely one facet of the infrastructure puzzle; discrepancies in low-level operator implementations (kernels) and parallelization strategies contribute equally to the mismatch. Attempting to comprehensively align these mechanisms across hardware and software stacks requires exhaustive, bespoke optimizations to handle myriad corner cases, which are notoriously labor-intensive. An alternative research line has emerged at the algorithmic level through heuristic interventions. \citet{li2025trust} found that masking the gradients generated by tokens or sentences with high discrepancy can bolster stability. Nevertheless, these methods harbor a subtle yet fundamental flaw: their primary focus is merely to artificially suppress the impact of noisy data on the training policy. They overlook the ultimate objective of LLM fine-tuning, which is to deliver superior performance on the inference policy deployed in production, not the training policy itself. Motivated by these profound limitations, this paper introduces a paradigm-shifting perspective:

\textcolor{purple}{\textbf{What if we endow the training policy with the inherent capacity to dynamically perceive and actively nullify its discrepancy with the inference policy?}} 

Intuitively, this can be formulated as a dual-objective optimization problem by translating the train-inference discrepancy into a step-wise penalty, thereby jointly maximizing the primary reward and minimizing the discrepancy between two policies. Yet, the core technical hurdle lies in formulating a penalty that harmonizes with the primary correctness reward while remaining robust across long-horizon generation. Through rigorous empirical and theoretical analysis, we identify that the absolute value of the probability difference at the token-level serves as an efficient penalty metric. Crucially, this penalty can be sensitive to the presence of individual tokens that cause extreme discrepancy, enabling fine-grained guidance. Employing this metric as a penalty acts as an endogenous signal, enabling the training policy to actively perceive the mismatch timely and initiate self-correction, which effectively eradicates training collapse. Digging deeper, we empirically observed a counter-intuitive phenomenon: policies exhibit a natural tolerance for minor discrepancy. Imposing aggressive penalties prematurely forces the policy to over-optimize for closing this negligible gap, thereby stifling the efficiency of performance gains. To elegantly exploit this \emph{discrepancy tolerance region}, we propose a fundamental paradigm shift in the LLM RL formulation: transitioning from a standard Markov Decision Process (MDP) to a \emph{Discrepancy-Constrained MDP} (\texttt{DCMDP}). Specifically, penalties are applied exclusively when the policy's behavior exceeds the tolerance boundary, triggering the performance-discrepency dual-objective optimization. Conversely, for negligible deviations within the boundary, the process gracefully degenerates into a pure reward-maximization objective. Furthermore, we introduce an ultra-lightweight operator to adaptively balance the importance of these two objectives. \texttt{DCMDP} not only resolves the policy discrepancy dilemma but also shatters existing performance bottlenecks as shown in \autoref{fig:1} (Left). 

Since our method elegantly frames the discrepancy between the training and deployment (inference) policies as a feedback-driven black-box optimization problem. Whether the deviation originates from low-level infrastructure discrepancies or high-level model architectural designs, our method handles them seamlessly. Even more exhilarating is the paradigm shift this capability unlocks. Our method satisfies a critical, yet unmet, training demand: highly efficient RL optimization for low-cost deployment policies (\autoref{fig:1}(Right)). This framework allows us to leverage the infrastructural advantages of a high-cost training setting (e.g., high-fidelity precision and lossless weight loading) to simulate and adapt to the constrained, low-cost conditions of the deployment setting. Ultimately, it exploits training strategies with high computational performance to find efficient solutions for inference strategies with resource-constrained, distributed applications deployment requirements.

To summarize, our primary contributions are threefold:
 \definecolor{blueviolet}{RGB}{138,43,226}
\newtcolorbox{insightblock}{
  colback=blueviolet!5,   
  colframe=blueviolet!50!black!50!,    
  boxrule=0.5mm,       
  arc=2mm,            
  left=0pt,           
  right=8pt,           
  top=8pt,            
  bottom=8pt,}

\begin{insightblock}
\begin{enumerate}[leftmargin=1.5em]\item We uncovered that incorporating a robust black-box penalty endows the training policy with the capacity to eradicate the pervasive risk of discrepancy-induced collapse autonomously.
\item We pioneer a paradigm shift by reformulating the LLM RL as \texttt{DCMDP} to optimally balance the dual objectives of maximizing reward and minimizing train-inference discrepancy.
\item We enable a heterogeneous training paradigm, paving the way for highly promising, resource-efficient model deployments in the real world.
\end{enumerate}
\end{insightblock}

\section{Preliminaries}
\paragraph{Reinforcement Learning for LLMs.}
The alignment of Large Language Models (LLMs) via RL is fundamentally grounded in the policy gradient theorem. Group Relative Policy Optimization (GRPO)~\citep{GRPO} as a widely used algorithm, it retains the core clipping mechanism of PPO \citep{PPO} but entirely eliminates the Critic network. Instead of relying on a learned value baseline, GRPO samples a group of $G$ responses $\{o_i\}_{i=1}^G$ for the same prompt $q$, and estimates the advantage $A_t$ by normalizing their respective rewards $\{r_i\}_{i=1}^G:
\hat{A}_{i,t} = \frac{r_i - \mathrm{mean}(\{r_i\}_{i=1}^G)}{\mathrm{std}(\{r_i\}_{i=1}^G)}$. This group-level normalization serves as a robust reward-shaping technique, effectively preserving gradient reliability even in sparse reward settings without the need for an external value estimator. The overall objective thus becomes:
\begin{equation}
\begin{aligned}
\mathcal{J}_{\mathrm{GRPO}}(\theta) =\ 
&\mathbb{E}_{\left[ 
  q \sim P(Q),\, \{o_i\}_{i=1}^G \sim \pi_{\theta_{\mathrm{old}}}(\cdot|q)
\right]} \\
&
  \frac{1}{G} \sum_{i=1}^G\, \frac{1}{|o_i|} \sum_{t=1}^{|o_i|}\ 
  \left\{\min\left(
    \gamma_{i,t}(\theta)\, \hat{A}_{i,t},\, 
    \mathrm{clip}\left(
      \gamma_{i,t}(\theta),\, 1{-}\epsilon,\, 1{+}\epsilon
    \right)\, \hat{A}_{i,t}
  \right)
\right\}.
\end{aligned}
\end{equation}
where $\gamma_t(\theta) = \frac{\pi_\theta(o_t|q, o_{<t})}{\pi_{\theta_{\mathrm{old}}}(o_t|q, o_{<t})}$ is the probability ratio. It is used by clipping to restrict the divergence between the current policy $\pi_\theta$ and the old policy $\pi_{\theta_{\mathrm{old}}}$. $\epsilon$ is the clipping range of trust region. 

\section{Untapped Potential: Low-Discrepancy Solutions exist in Decision Space} \label{sec:3.1}
To algorithmically reconstruct the learning objective such that the training policy converges to an ideal parameter landscape that exhibits both high performance and low train-inference discrepancy (a region favored by the deployment policy), a fundamental prerequisite must be met: the decision space of the training policy must inherently contain latent solutions that simultaneously satisfy these dual criteria. To this end, we conduct an empirical case study to validate this underlying hypothesis.

Specifically, we employ GRPO \citep{GRPO}, a widely adopted baseline in LLM RL, to optimize the 
\texttt{Deepseek-Distilled-Qwen-2.5-1.5b} model. We deliberately select mathematical reasoning tasks for evaluation that naturally possess a vast solution space composed of diverse reasoning paths. Furthermore, these tasks pose significant challenges for a 1.5B-parameter model \citep{Liu2025AsymmetricPP}, leaving substantial headroom for improvement and thus serving as a sensitive, precise barometer for RL efficacy. During the experiments, we explicitly control the policy's exploration propensity within the decision space by sweeping the sampling temperature $t=\{0.3,0.7,1.0\}$. As shown in \autoref{fig:2} (Left, Middle), both the asymptotic performance gains and the training stability during RL optimization exhibit a pronounced positive correlation with the degree of exploration.This provides the hypothesis: \emph{there indeed exist valid trajectories characterized by both low inter-policy mismatch and high reward.}

However, merely relying on elevated temperatures to blindly inject sampling diversity is not a fundamental cure, and the training will eventually collapse. To efficiently and stably guide the agent toward these optimal landscapes amidst an astronomically large decision space, an accurate and dense feedback signal is imperative. This necessitates the formulation of an idealized penalty metric.

\begin{figure}[h]
\centering

    \includegraphics[width=\linewidth]{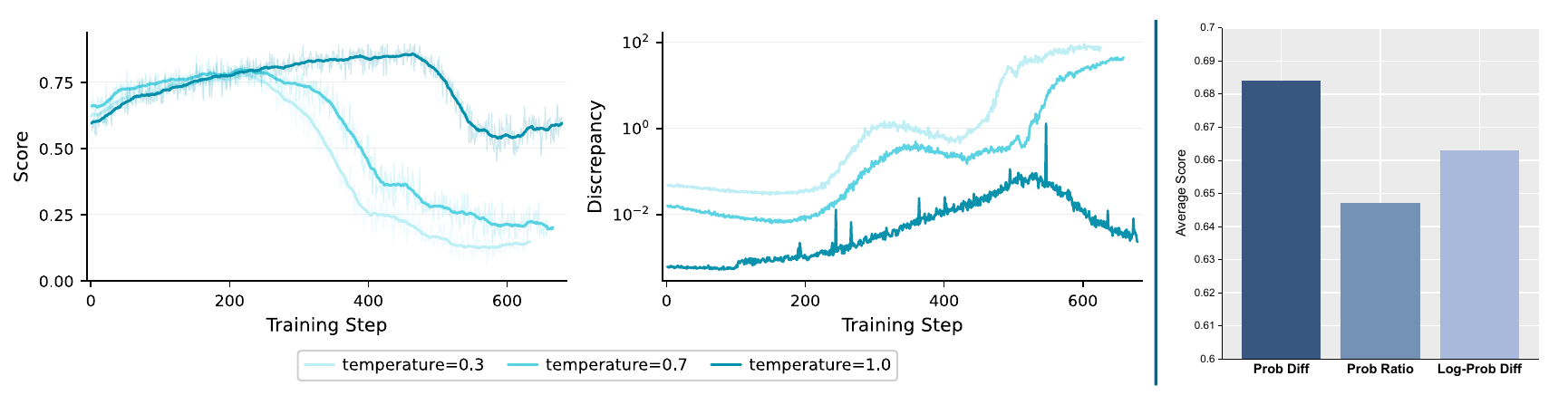}
    \vspace{-0.5cm}
    \caption{(Left):\textbf{Training curve}. Comprehensive performance of GRPO-driven policies with different temperatures on six mathematical scenarios. (Middle): \textbf{Discrepancy ratio}. The corresponding training-inference difference ratio during training. (Right): \textbf{Probability Difference penalty achieves the best performance}. average@32 scores on six benchmarks obtained using different penalties}
    \label{fig:2}
    \vspace{-0.5cm}
\end{figure}
\section{Enabling Autonomous Discrepancy Compensation via a Magic Penalty}
Building upon the empirical verification in \autoref{sec:3.1}, which confirms the latent existence of optimization landscapes satisfying both high task performance and low train-inference discrepancy, this section aims to formulate a reliable mismatch-aware penalty for exploration guidance. Given the same model weights $\theta_{\text{old}}$ and the same generated trajectory $\tau=\{o_1,\dots,o_T\}$, the training and inference backends may instantiate two slightly different effective policy distributions, denoted as $\pi_{\theta_{\text{old}}}^{\text{Train}}(\cdot|q,o_{<t})$ and $\pi_{\theta_{\text{old}}}^{\text{Inf}}(\cdot|q,o_{<t})$, respectively. The central question is therefore not merely whether such a discrepancy exists, but how it should be measured and penalized such that the resulting signal faithfully reflects deployment-relevant mismatch while remaining compatible with stable reinforcement learning.

We begin with three practical token-level discrepancy signals, where
$p_t^{\text{Train}}=\pi_{\theta_{\text{old}}}^{\text{Train}}(o_t|q,o_{<t})$
and
$p_t^{\text{Inf}}=\pi_{\theta_{\text{old}}}^{\text{Inf}}(o_t|q,o_{<t})$:
\begin{equation}
\begin{aligned}
\texttt{Probability Difference}\;:\qquad
\delta_t^{\mathrm{diff}}
&=
p_t^{\mathrm{Train}} - p_t^{\mathrm{Inf}}
\\
\texttt{Probability Ratio}\;:\qquad
\delta_t^{\mathrm{ratio}}
&=
\frac{p_t^{\mathrm{Train}}}{p_t^{\mathrm{Inf}}}
\\
\texttt{Log-Probability Difference}\;:\qquad
\delta_t^{\mathrm{log}}
&=
\log p_t^{\mathrm{Train}} - \log p_t^{\mathrm{Inf}}
\end{aligned}
\end{equation}

\vspace{-0.2cm}
Although all three quantities vanish when the two backends agree, they induce substantially different regularization geometries. The probability difference $\delta_t^{\mathrm{diff}}$ constrains absolute drift in token probability. As a result, increasing the probability of a low-probability token may incur only a small penalty as long as the absolute change remains limited, even if the relative amplification is large. This makes the induced regularization highly plastic, but also potentially tail-seeking: when a low-probability token receives a large advantage, the optimizer can aggressively shift probability mass toward it. In contrast, the probability ratio $\delta_t^{\mathrm{ratio}}$ constrains multiplicative drift. For tokens with small $p_t^{\mathrm{Inf}}$, even a modest absolute increase can correspond to a large ratio change, leading to a much stronger penalty. Ratio-based regularization therefore more strictly preserves the support of the deployed inference policy and discourages abrupt amplification of tail tokens. 
However, this conservatism can also become restrictive: if genuinely beneficial tokens initially reside in the low-probability region of the inference policy, the ratio penalty may suppress the very adaptations needed to improve performance.
 $\delta_t^{\mathrm{log}}$ provides an intermediate geometry. In the small-drift regime, when $p_t^{\text{Train}}$ remains close to $p_t^{\text{Inf}}$, we have $
\log \frac{p_t^{\text{Train}}}{p_t^{\text{Inf}}}
=
\log\left(1+\frac{p_t^{\text{Train}}-p_t^{\text{Inf}}}{p_t^{\text{Inf}}}\right)
\approx
\frac{p_t^{\text{Train}}-p_t^{\text{Inf}}}{p_t^{\text{Inf}}},
$ which shows that $\delta_t^{\text{log}}$ locally behaves like a ratio-style penalty and therefore inherits its stabilizing effect against relative mismatch. Outside this local regime, however, the logarithmic transformation grows more slowly than the raw ratio, avoiding excessive punishment of all non-negligible deviations. Theoretically, we further analyze how these three discrepancy parameterizations induce different local optimal policies, and use a one-state toy experiment to illustrate how advantage-weighted updates reshape the resulting policy distributions; details are deferred to Appendix~\ref{app:optimal_policy_toy}.

We empirically compare these three shaped discrepancy penalties under identical post-training configurations. As illustrated in \autoref{fig:2} (Right), $\delta_t^{\text{dif}}$ achieves the strongest overall performance. It consistently improves the primary verifiable reward while suppressing train-inference discrepancy. This empirical ordering indicates that effective discrepancy-aware regularization should correct deployment-relevant drift without overly constraining reward-driven exploration. $\delta_t^{\mathrm{diff}}$ offers a more plastic penalty: it narrows train-inference discrepancy while preserving the policy’s ability to shift probability mass toward useful reasoning trajectories.

\vspace{-0.1cm}
\paragraph{The Choice of Quantization Granularity}A natural alternative is to first aggregate the discrepancy over the entire generated response and then apply the penalty at the sequence level. Although such sequence-level calculations are structurally aligned with response-level rewards, they are intrinsically coarse. In long-horizon reasoning tasks, train-inference discrepancy is often highly sparse and token-localized: a small number of tokens may exhibit substantial backend disagreement and exert a disproportionate influence on subsequent generation quality. Once averaged over the whole sequence, however, these critical tokens can be diluted by a large number of benign tokens with negligible discrepancy \citep{liu2025bingo}. Therefore, we instead compute the discrepancy penalty at the token level. This design preserves the fine-grained mismatch structure along the trajectory. Consequently, tokens with large train-inference mismatch receive direct corrective feedback.

\vspace{-0.1cm}
\paragraph{Dual-objective RL}
We incorporate the token-level log-probability discrepancy into the reward-maximized RL as an additional optimization objective. In practice, we use GRPO as an example. Given a prompt $q$ and a group of responses $\{o_i\}_{i=1}^{G}$ sampled from the old policy, the standard GRPO objective encourages reward-driven policy improvement through the clipped importance ratio. In contrast, our formulation additionally penalizes the discrepancy measured on each generated token, thereby guiding the learned policy toward solutions that are not only reward-favorable but also faithfully supported by the inference policy. Formally, the per-token penalty is
\begin{equation}
\delta_{i,t}
=
\left|
\pi_{\theta_{\mathrm{old}}}^{\mathrm{Train}}(o_{i,t}\mid q,o_{i,<t})
-
\pi_{\theta_{\mathrm{old}}}^{\mathrm{Inf}}(o_{i,t}\mid q,o_{i,<t})
\right|.
\end{equation}
Since the group-relative advantage $\hat{A}_{i,t}$ of GRPO is broadcast to every token in the response, we directly inject the discrepancy penalty at the advantage level, yielding a dual-objective advantage $
\textcolor{purple}{\hat{A}_{i,t}^{\mathrm{Dual}}}
=
\hat{A}_{i,t} - \textcolor{purple}{\delta_{i,t}},
$
and the Dual-GRPO objective becomes:
\begin{equation}
\begin{aligned}
\mathcal{J}_{\mathrm{\textcolor{purple}{Dual\text{-GRPO}}}}(\theta) \
&=\mathbb{E}_{q \sim P(Q),\, \{o_i\}_{i=1}^G \sim \pi_{\theta_{\mathrm{old}}}^{\mathrm{Inf}}(\cdot|q)} \\
&\frac{1}{G} \sum_{i=1}^G\, \frac{1}{|o_i|} \sum_{t=1}^{|o_i|}\
  \min\left(
    \gamma_{i,t}(\theta)\, \textcolor{purple}{\hat{A}_{i,t}^{\mathrm{Dual}}},\,
    \mathrm{clip}\left(
      \gamma_{i,t}(\theta),\, 1-\varepsilon,\, 1+\varepsilon
    \right)\, \textcolor{purple}{\hat{A}_{i,t}^{\mathrm{Dual}}}
  \right).
\end{aligned}
\end{equation}

\section{Transitioning to Discrepancy Constrained-MDP for Stable Post-training}

\begin{wrapfigure}{r}{0.48\textwidth}
\centering
\vspace{-0.8cm}
    \includegraphics[width=\linewidth]{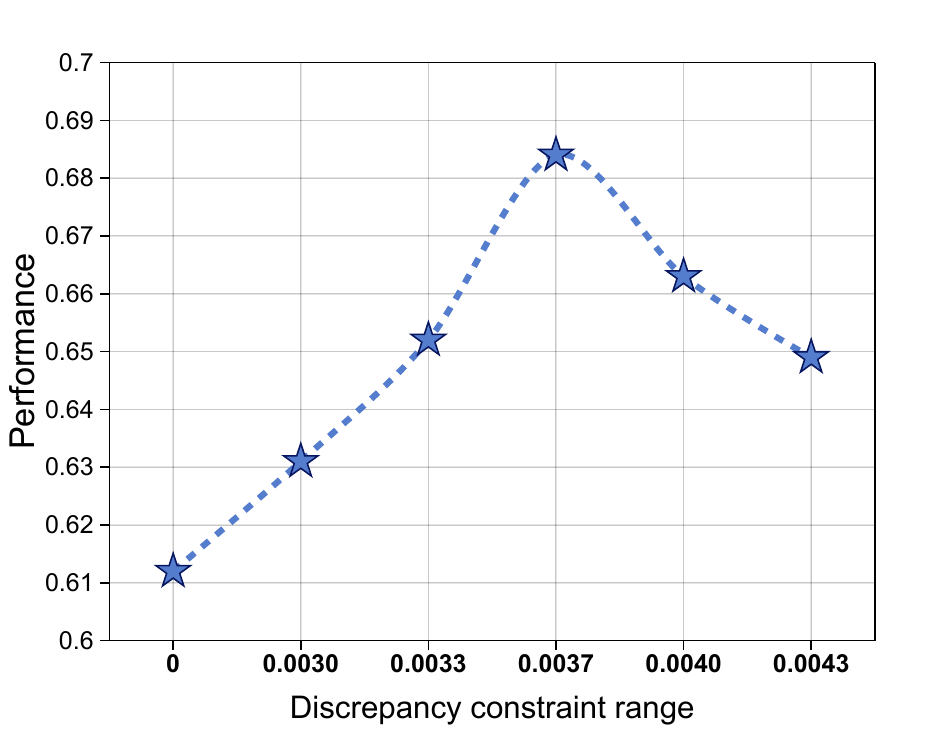}
    \vspace{-0.9cm}
    \caption{\textbf{Performance under various tolerance ranges.} As the region expands, the performance gradually increases, but a loose range leads to a degradation. Use the setup in \autoref{fig:2}.}
    \label{fig:CMDP_2}
    \vspace{-0.6cm}
\end{wrapfigure}

While the reliable performance-discrepancy dual-objective formulation (Dual-GRPO) eradicates discrepancy-induced training collapses, we find a significant limitation of such a learning objective from a practice perspective. Compared to the naive single-objective GRPO, Dual-GRPO suffers from severely degraded learning efficiency during the early stages of training and ultimately converges to a sub-optimal performance ceiling. We hypothesize that this limitation stems from over-regularization: when the inter-policy discrepancy is initially minute, an indiscriminate and strict penalty forces the agent to prematurely prioritize alignment with the inference policy at the expense of its primary objective, i.e., maximizing the verifiable task reward, thereby stifling the overall learning momentum. This observation inevitably begs the question: Does the training policy possess an inherent resistance, or a discrepancy tolerance region, against minor training-inference deviations?

\begin{wrapfigure}{r}{0.48\textwidth}
\centering
\vspace{-0.55cm}
    \includegraphics[width=\linewidth]{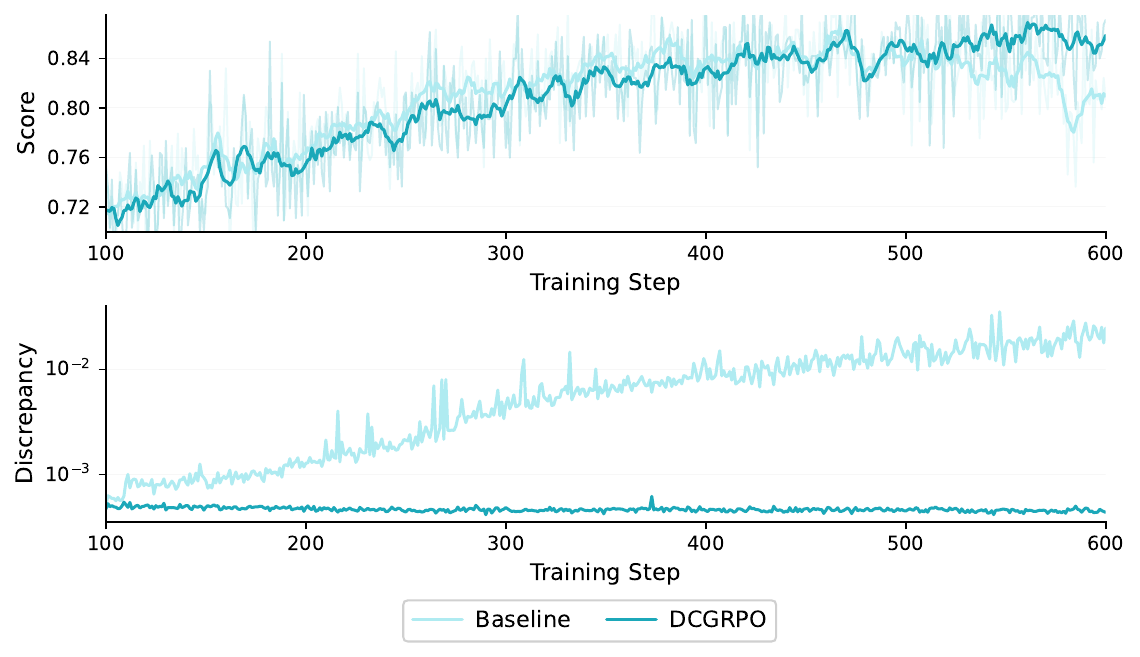}
    \caption{Compared to the naive GRPO, \texttt{DC-GRPO} controls the train-inference discrepancy at a small value, achieves more stable training and overcomes collapse, ultimately achieving better progressive gains. The experiments use the same setting in \autoref{fig:2}.}
    \label{fig:5}
    \vspace{-0.3cm}
\end{wrapfigure} 

To validate this conjecture, we introduce a relaxed, threshold-based truncation to the penalty, activating the discrepancy feedback only when the deviation surpasses a predefined margin $c$. As depicted in \autoref{fig:CMDP_2}, implementing this leniency improves the policy's learning efficiency and performance. In contrast, punishments that are too lenient can have a negative effect. This evidence validates the existence of the \textit{discrepancy tolerance region}.

Driven by these empirical observations, it becomes evident that a naive, unconstrained scalarization of the two objectives is highly inefficient. Instead, a more structurally sound and mathematically elegant modeling paradigm for LLM RL emerges: the \textit{Discrepancy-Constrained Markov Decision Process} (\texttt{DCMDP}). The core innovation of the \texttt{DCMDP} lies in the introduction of an explicit discrepancy constraint in addition to the reward-maximization objective. Specifically, it constrains the black-box train-inference discrepancy within the empirically identified \textit{discrepancy tolerance region}, allowing harmless microscopic deviations while preventing large backend-induced distributional shifts in either direction. In this way, the \texttt{DCMDP} preserves the exploration flexibility of the training policy inside the tolerated region. Formally, building upon the token-level absolute probability discrepancy, the \texttt{DCMDP} objective can be formulated as:
\begin{align}
    \max_{\theta} \quad
    & \mathcal{J}(\theta)
    =
    \mathbb{E}_{\tau \sim \pi_\theta^{\mathrm{Train}}}
    \left[
    \sum_{t=1}^{|\tau|}
    r(o_t \mid q,o_{<t})
    \right] \\
    \text{s.t.} \quad
    & \frac{1}{|\tau|}\sum_{t=1}^{|\tau|}\left|
\pi_{\theta_{\mathrm{old}}}^{\mathrm{Train}}
(o_t \mid q,o_{<t})
-
\pi_{\theta_{\mathrm{old}}}^{\mathrm{Inf}}
(o_t \mid q,o_{<t})
\right| \le c,
    \quad
    \forall \tau \sim \pi_\theta^{\mathrm{Train}} .
    \label{eq:constraint}
\end{align}

\begin{wrapfigure}{r}{0.48\textwidth}
\centering
\vspace{-0.55cm}
    \includegraphics[width=\linewidth]{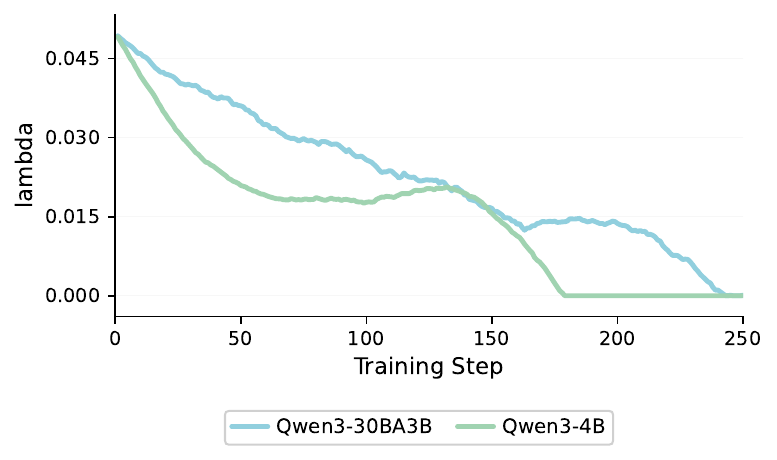}
    \vspace{-0.7cm}
    \caption{Depending on the adaptive optimization ability of the Lagrangian operator, we find that its behavior for different models is different, which is hard to achieve by heuristics.}
    \label{fig:4}
    \vspace{-0.8cm}
\end{wrapfigure}

$r(\cdot)$ denotes the verifiable reward, and $c \ge 0$ specifies the radius of the symmetric two-bounded discrepancy tolerance region. The discrepancy penalty only triggers when the train-inference discrepancy remains within a centered tolerance interval, whereas deviations beyond either boundary, i.e., both $\pi^{\mathrm{Train}} > \pi^{\mathrm{Inf}}$ and $\pi^{\mathrm{Train}} < \pi^{\mathrm{Inf}}$, are explicitly restricted through the absolute-value operator.

\paragraph{Solving \texttt{DCMDP} via Lagrangian Relaxation.} 
We incorporate the Lagrangian Relaxation method to achieve an adaptive and smooth equilibration of the dual objectives \citep{achiam2017constrained,ji2024omnisafe}. Specifically, during RL training, a Lagrangian multiplier $\lambda$ acts as a dynamic penalty weight, updated iteratively based on the average discrepancy measured across the sampled batch at each time step. When the inter-policy divergence spikes and violates the tolerance boundary, $\lambda$ progressively scales up the penalty strength to forcefully pull the policy back into the safe zone. Conversely, when the discrepancy drops below the threshold, $\lambda$ gently anneals toward zero. This responsive mechanism is associated with a dynamic buffer for calculation. Ultimately, using GRPO as the algorithmic backbone, our framework for solving the \texttt{DCMDP} can be cast as a Min-Max optimization problem:
\begin{equation}\label{eq:lagrangian_obj}
    \max_{\theta} \min_{\lambda \in [\lambda_{\min}, \lambda_{\max}]} \mathcal{L}(\theta, \lambda) = \mathcal{J}_{\mathrm{GRPO}}(\theta) - \lambda \left( \mathbb{E}_{q, \{o_i\}_{i=1}^G,\,t} \left[ \delta_{i,t} \right] - c \right),
\end{equation}
where $\delta_{i,t}=\left|\pi_{\theta_{\mathrm{old}}}^{\mathrm{Train}}(o_{i,t}\mid q,o_{i,<t})-\pi_{\theta_{\mathrm{old}}}^{\mathrm{Inf}}(o_{i,t}\mid q,o_{i,<t})\right|$ is the token-level absolute probability discrepancy, $\lambda$ is a learnable Lagrangian multiplier bounded within $[\lambda_{\min},\lambda_{\max}]$ that adaptively controls the penalty strength, and $c$ specifies the discrepancy tolerance boundary. By the linearity of expectation and the policy gradient theorem, optimizing this Lagrangian objective with respect to the policy parameters $\theta$ is mathematically equivalent to dynamically shaping the token-level advantage. In our practical implementation, at training iteration $k$, the penalty is injected directly into the group-normalized advantage:
\begin{align}
\label{eq:lagrangian_reward}
    \textcolor{Mulberry}{\hat{A}_{i,t}^{\texttt{DCMDP}}} &= \hat{A}_{i,t} - \textcolor{Mulberry}{\underbrace{\lambda_k \cdot \delta_{i,t}}_{\text{Adaptive Penalty}}},\quad\text{where}\quad \hat{A}_{i,t} = \frac{r_i^{\text{task}} - \mathrm{mean}(\{r_i^{\text{task}}\}_{i=1}^G)}{\mathrm{std}(\{r_i^{\text{task}}\}_{i=1}^G)},\\
    \mathcal{J}_{ \textcolor{Mulberry}{\texttt{DC-GRPO}}}(\theta) &=\
\mathbb{E}_{q \sim P(Q),\, \{o_i\}_{i=1}^G \sim \pi_{\theta_{\mathrm{old}}}^{\mathrm{Inf}}(\cdot|q)} \\
& \frac{1}{G} \sum_{i=1}^G\,
  \frac{1}{|o_i|} \sum_{t=1}^{|o_i|}
  \min\left(
    \gamma_{i,t}(\theta)\, \textcolor{Mulberry}{\hat{A}_{i,t}^{\texttt{DCMDP}}},\,
    \mathrm{clip}\left(
      \gamma_{i,t}(\theta),\, 1{-}\epsilon_{\text{low}},\, 1{+}\epsilon_{\text{high}}
    \right)\, \textcolor{Mulberry}{\hat{A}_{i,t}^{\texttt{DCMDP}}}
  \right).
\end{align}
The policy $\pi_\theta$ is then updated via the standard GRPO clipped surrogate using the penalized advantage $\hat{A}_{i,t}^{\texttt{DCMDP}}$. Simultaneously, to solve the minimization problem with respect to the dual variable, the multiplier $\lambda$ is updated via dual gradient ascent at the end of each iteration based on the batch-average token-level discrepancy:
\begin{equation}\label{eq:lambda_update}
    \lambda_{k+1} = \Pi_{[\lambda_{\min},\lambda_{\max}]}\!\left( \lambda_k + \eta_\lambda \left( \frac{1}{\sum_{i=1}^{B}|o_i|} \sum_{i=1}^{B}\sum_{t=1}^{|o_i|} \delta_{i,t} - c \right) \right),
\end{equation}

where $\eta_\lambda$ is the dual learning rate, $B$ is the number of trajectories in the current batch, and $\Pi_{[\lambda_{\min},\lambda_{\max}]}(\cdot)$ is the projection onto the admissible range of the multiplier.

\section{Experiments}
\label{sec:experiments}

\subsection{Setup}
\label{subsec:exp_setup}

\paragraph{Models and Baselines.}

We conduct experiments on two open-source LLMs spanning both dense and Mixture-of-Experts (MoE) architectures: \QEightB~\citep{Qwen3} and \QThirtyB~\citep{Qwen3} (30B total parameters with 3B activated per token). These two models cover substantially different training regimes, e.g., FSDP2 for the dense model and Megatron-LM~\citep{Megatron-LM} with tensor/expert parallelism for the MoE model, allowing us to comprehensively evaluate the robustness of \ourmethod across infrastructures that are prone to different train-inference discrepancy patterns. We compare \ourmethod against the vanilla GRPO~\citep{GRPO} baseline, which shares the identical training pipeline, data, and hyperparameters except for the discrepancy-constrained penalty, thereby isolating the effect of the proposed constraint.

\paragraph{Evaluation.}

We evaluate all methods on six widely-used mathematical reasoning benchmarks: AIME24~\citep{AIME24}, AIME25~\citep{AIME25}, AMC23~\citep{AMC23}, MATH-500~\citep{PRM800K}, Minerva Math~\citep{Minerva-Math}, and OlympiadBench~\citep{OlympiadBench}. Evaluation uses vLLM 0.11.0~\citep{vLLM} with a maximum generation length of 32,768 tokens, temperature 0.7, and top-$p$ of 1.0. We report avg@$K$ with $K=32$ for AIME24/25 and AMC23, and $K=4$ for MATH-500, Minerva, and OlympiadBench.

\paragraph{Implementation Details.}
\label{para:implement}
We train all methods on the DAPO-Math-17K~\citep{DAPO} dataset using verifiable mathematical rewards computed by \texttt{math-verify}. The maximum response length is set to 8,192 tokens. For each prompt we sample $G=8$ rollouts with temperature 1.0 and top-$p$ of 1.0. The train batch size is 64 prompts, with a PPO mini-batch size of 16. We use AdamW with learning rate $1\!\times\!10^{-6}$, betas $(0.9, 0.95)$, weight decay $0$, and optimizer epsilon $1\!\times\!10^{-15}$ following~\citep{Qi2025DefeatingTT}. We use $\varepsilon_{\text{high}}{=}0.24$ and disable the KL loss. The dense \QEightB is trained on $1\!\times\!8$ NVIDIA GPUs under the FSDP2 backend, while the MoE \QThirtyB is trained on $4\!\times\!8$ GPUs under the Megatron backend (TP$=$2, EP$=$8, rollout TP$=$4) with BF16 weights.

The discrepancy penalty is instantiated as the token-level absolute probability difference $\delta_{i,t}=|\pi_{\theta_{\mathrm{old}}}^{\mathrm{Train}}-\pi_{\theta_{\mathrm{old}}}^{\mathrm{Inf}}|$ computed from the rollout (vLLM, inference backend) and recomputation (training backend) log-probabilities on the exact same token sequences. The Lagrangian multiplier $\lambda$ is initialized at $\lambda_0=0.1$ for \QEightB and $\lambda_0=0.05$ for \QThirtyB, with dual learning rate $\eta_\lambda=1.0$. The discrepancy tolerance budget $c$ is set to $0.0037$ for \QEightB and $0.0050$ for \QThirtyB, reflecting the empirically higher baseline mismatch of MoE rollouts due to expert routing non-determinism. The penalty is injected on the token-level advantage as in Eq.~\eqref{eq:lagrangian_reward}, and $\lambda$ is updated once per training step.

\subsection{Main Results}

\paragraph{BF16 main results.}
Table~\ref{tab:main_bf16} reports mathematical reasoning results for \QEightB and \QThirtyB under the setting of BF16. Across both backbones, \ourmethod consistently outperforms the GRPO baseline on every benchmark and improves the overall average performance, demonstrating that enforcing discrepancy-tolerance constraints not only preserves, but can even enhance the primary reward-driven optimization signal. These results suggest that DCMDP provides a unified framework for mitigating both architecture-induced and infrastructure-induced instability. This highlights an algorithmic advantage of our approach: instead of explicitly tracing and modeling the underlying source of train-inference mismatch, \texttt{DCMDP} treats the mismatch as a black-box optimization.

\begin{table*}[!ht]
\centering
\caption{Evaluation results on mathematical benchmarks under the BF16 training/inference regime. The results of \ourmethod are \colorbox{\mycolor}{shaded} and the highest values are \textbf{bolded}.}
\vspace{0.1cm}
\resizebox{\textwidth}{!}{
\begin{tabular}{lccccccc}
\toprule
{\textbf{Method}} & {\textbf{AIME24}} & {\textbf{AIME25}} & {\textbf{AMC23}} & {\textbf{MATH-500}} & {\textbf{Minerva}} & {\textbf{Olympiad}} & {\textbf{Avg.}} \\
\midrule
\textbf{\QEightB}  &  &  &  &  &  &  &  \\
$\drsh$ GRPO       & 58.8 & 40.0 & 90.4 & 93.7 & \textbf{51.0} & 70.1 & 67.3 \\
\rowcolor{cyan!10} $\drsh$ \ourmethod & \textbf{63.2} & \textbf{52.5} & \textbf{91.8} & \textbf{94.2} & 50.1 & \textbf{74.4} & \textbf{71.0} \\
\midrule
\textbf{\QThirtyB} &  &  &  &  &  &  &  \\
$\drsh$ GRPO       & 60.2 & 43.6 & 91.4 & 95.5 & 52.0 & 71.3 & 69.0 \\
\rowcolor{cyan!10} $\drsh$ \ourmethod & \textbf{66.8} & \textbf{49.1} & \textbf{91.9} & \textbf{95.8} & \textbf{52.9} & \textbf{74.5} & \textbf{71.8} \\
\bottomrule
\end{tabular}%
}
\label{tab:main_bf16}%
\end{table*}%

\paragraph{Heterogeneous FP8 deployment.}
A central motivation of \ourmethod is to enable \emph{heterogeneous} training: the policy is updated under a high-fidelity training backend (BF16) while being optimized for a low-cost deployment backend (FP8 vLLM). Table~\ref{tab:main_fp8} reports the results in this heterogeneous setting, where GRPO suffers noticeable degradation or outright collapse on several benchmarks due to the amplified train-inference gap, whereas \ourmethod preserves (and often exceeds) the BF16-only performance by natively closing the backend gap. This validates the core claim of our method: by framing the train-inference gap as a black-box feedback signal, \ourmethod turns infrastructure-level heterogeneity from a liability into a controllable optimization knob.

\begin{table*}[!ht]
\centering
\caption{Evaluation results on mathematical benchmarks under the heterogeneous training (BF16) + FP8 deployment regime. The results of \ourmethod are \colorbox{\mycolor}{shaded} and the highest values are \textbf{bolded}.}
\vspace{0.2cm}
\resizebox{\textwidth}{!}{
\begin{tabular}{lccccccc}
\toprule
{\textbf{Method}} & {\textbf{AIME24}} & {\textbf{AIME25}} & {\textbf{AMC23}} & {\textbf{MATH-500}} & {\textbf{Minerva}} & {\textbf{Olympiad}} & {\textbf{Avg.}} \\
\midrule
\textbf{\QEightB} &  &  &  &  &  &  &   \\
$\drsh$ GRPO      &  53.2 & 35.2 & 89.5 & 92.9 & 50.0 & 67.4 & 64.7 \\
\rowcolor{cyan!10} $\drsh$ \ourmethod & \textbf{55.0} & \textbf{44.1} & \textbf{90.2} & \textbf{94.4} & \textbf{54.0} & \textbf{70.9} & \textbf{68.1} \\
\midrule
\textbf{\QThirtyB} \\
$\drsh$ GRPO & 52.4 & 37.5 & 87.3 & 93.8 & 49.6 & 70.5 & 65.2 \\
\rowcolor{cyan!10} $\drsh$ \ourmethod & \textbf{56.6} & \textbf{41.2} & \textbf{91.8} & \textbf{94.0} & \textbf{50.5} & \textbf{72.0} & \textbf{67.7} \\
\bottomrule
\end{tabular}%
}
\label{tab:main_fp8}%
\end{table*}%

\subsection{Ablation Study}

To disentangle the contribution of each design choice of \ourmethod, we conduct ablations on \QEightB by varying the three core hyperparameters of the Lagrangian relaxation. All other training settings are kept identical to the main experiments.

\paragraph{Initial value of $\lambda$.}
We sweep $\lambda_0\in\{0.05, 0.1, 0.2\}$. Overall, \ourmethod is not very sensitive to $\lambda_0$: all three settings converge to competitive and stable performance, because the dual variable $\lambda$ is updated online and quickly adjusts to the prevailing discrepancy level. The remaining differences are confined to early training: an overly small $\lambda_0$ (e.g., $0.05$) provides only a weak initial penalty and therefore offers limited early containment of backend-induced spikes, whereas an overly large $\lambda_0$ (e.g., $0.2$) over-regularizes the policy at the very beginning and slows down early-stage reward maximization. $\lambda_0{=}0.1$ achieves a good middle-ground and is used as our default. As shown in Table~\ref{tab:ablation_lambda_init}, the final-performance gap across settings is small.

\begin{table*}[!ht]
\centering
\caption{Ablation on the initial value of $\lambda$ for \QEightB. All other settings follow the main experiments ($c{=}0.0037$, $\eta_\lambda{=}1.0$). Results are \colorbox{\mycolor}{shaded} for the default.}
\vspace{0.2cm}
\resizebox{\textwidth}{!}{
\begin{tabular}{lccccccc}
\toprule
{$\lambda_0$} & {\textbf{AIME24}} & {\textbf{AIME25}} & {\textbf{AMC23}} & {\textbf{MATH-500}} & {\textbf{Minerva}} & {\textbf{Olympiad}} & {\textbf{Avg.}} \\
\midrule
$0.05$ & 59.8 & 43.8 & 91.5 & 94.7 & 52.0 & 71.8 & 68.9 \\
\rowcolor{cyan!10} $0.10$ (default) & 63.2 & 52.5 & 91.8 & 94.2 & 50.1 & 74.4 & 71.0 \\
$0.20$ & 64.9 & 45.3 & 92.8 & 95.4 & 52.0 & 72.0 & 70.4 \\
\bottomrule
\end{tabular}%
}
\label{tab:ablation_lambda_init}%
\end{table*}%

\paragraph{Dual learning rate of $\lambda$.}
We compare $\eta_\lambda\in\{0.5, 1.0, 1.5\}$. Similar to $\lambda_0$, \ourmethod remains robust to a broad range of dual learning rates. Only when $\eta_\lambda$ is very small (e.g., well below $0.5$) does the multiplier react too slowly to discrepancy spikes, allowing constraint violations to accumulate and noticeably hurting final performance. Moderately larger values such as $\eta_\lambda{=}1.5$ still work well: although larger $\eta_\lambda$ in principle induces faster oscillations of $\lambda$ around the tolerance boundary, the projection onto $[\lambda_{\min},\lambda_{\max}]$ and the smoothing provided by the batch-averaged discrepancy keep the training stable in practice. $\eta_\lambda{=}1.0$ is chosen as our default for its clean convergence behavior. Detailed numbers are shown in Table~\ref{tab:ablation_lambda_lr}.

\begin{table*}[!ht]
\centering
\caption{Ablation on the dual learning rate $\eta_\lambda$ for \QEightB. All other settings follow the main experiments ($c{=}0.0037$, $\lambda_0{=}0.1$). Results are \colorbox{\mycolor}{shaded} for the default.}
\vspace{0.2cm}
\resizebox{\textwidth}{!}{
\begin{tabular}{lccccccc}
\toprule
{$\eta_\lambda$} & {\textbf{AIME24}} & {\textbf{AIME25}} & {\textbf{AMC23}} & {\textbf{MATH-500}} & {\textbf{Minerva}} & {\textbf{Olympiad}} & {\textbf{Avg.}} \\
\midrule
$0.5$ & 61.5 & 43.6 & 90.5 & 94.4 & 50.0 & 71.5 & 68.6 \\
\rowcolor{cyan!10} $1.0$ (default) & 63.2 & 52.5 & 91.8 & 94.2 & 50.1 & 74.4 & 71.0 \\
$1.5$ & 63.1 & 46.0 & 94.2 & 95.7 & 52.5 & 72.9 & 70.7 \\
\bottomrule
\end{tabular}%
}
\label{tab:ablation_lambda_lr}%
\end{table*}%

\paragraph{Tolerance budget $c$.}
When $c$ is set too small, the batch-average discrepancy persistently exceeds $c$ throughout training, causing $\lambda$ to monotonically increase and saturate at its upper bound; the resulting overly strong penalty dominates the advantage and impedes the primary reward-maximization objective of RL. Conversely, when $c$ is too large, the discrepancy rarely crosses the boundary and $\lambda$ keeps decreasing toward $\lambda_{\min}$, so the penalty effectively vanishes and \ourmethod degenerates to vanilla GRPO, losing its constraint-control capability. The middle setting $c{=}0.0037$, yields the best downstream performance (Table~\ref{tab:ablation_c}).

\begin{table*}[!ht]
\centering
\caption{Ablation on the tolerance budget $c$ for \QEightB. All other settings follow the main experiments ($\lambda_0{=}0.1$, $\eta_\lambda{=}1.0$). Results are \colorbox{\mycolor}{shaded} for the default.}
\vspace{0.2cm}
\resizebox{\textwidth}{!}{
\begin{tabular}{lccccccc}
\toprule
{$c$} & {\textbf{AIME24}} & {\textbf{AIME25}} & {\textbf{AMC23}} & {\textbf{MATH-500}} & {\textbf{Minerva}} & {\textbf{Olympiad}} & {\textbf{Avg.}} \\
\midrule
$0.0030$ & 61.3 & 49.5 & 92.3 & 95.2 & 52.3 & 72.9 & 70.6 \\
$0.0035$ & 63.0 & 48.8 & 92.8 & 94.9 & 52.1 & 72.7 & 70.7 \\
\rowcolor{cyan!10} $0.0037$ (default) & 63.2 & 52.5 & 91.8 & 94.2 & 50.1 & 74.4 & 71.0 \\
$0.0040$ & 62.9 & 44.8 & 92.3 & 95.2 & 51.0 & 74.1 & 70.1 \\
\bottomrule
\end{tabular}%
}
\label{tab:ablation_c}%
\end{table*}%

\section{Conclusion}

In this paper, we reformulate LLM RL as a \emph{Discrepancy-Constrained MDP} (\texttt{DCMDP}) to tackle the long-standing train-inference discrepancy that underlies unpredictable collapses: using the token-level absolute probability difference as a black-box discrepancy signal, we inject an adaptive Lagrangian penalty into the advantage that activates only when the deviation exceeds an empirically-identified \emph{tolerance region}, thereby preserving reward-driven exploration inside the safe zone while enforcing deployment-faithful optimization outside of it; experiments on \QEightB and \QThirtyB across various benchmarks show that our method outperforms GRPO under the wildely used BF16 training setup, and the heterogeneous BF16-training + FP8-deployment regime, validating that train-inference gap can be mitigated at algorithm level.

\paragraph{Limitations.}
Our empirical scope is limited to BF16 and FP8 precisions and extending \texttt{DCMDP} to other backend combinations is left for future work. In addition, although we verify the effectiveness of our f method on large-sized models as well as advanced architecture models, e.g., MoE, there is a lack of evaluation on more open source model families and analysis on different model sizes. We will further improve the related research in the future.

\clearpage
\bibliographystyle{plainnat}  
\bibliography{ref}


\newpage
\appendix

\section{Full Algorithm}
\label{app:algorithm}

We summarize the full training procedure of \ourmethod in Algorithm~\ref{alg:algorithm}. At each iteration, the same rollout policy weights $\theta_{\mathrm{old}}$ are executed under two backends: the high-throughput inference backend $\pi_{\theta_{\mathrm{old}}}^{\mathrm{Inf}}$ (e.g., vLLM) that generates the sampled trajectories, and the high-fidelity training backend $\pi_{\theta_{\mathrm{old}}}^{\mathrm{Train}}$ (e.g., FSDP2/Megatron) that recomputes the per-token probabilities. The absolute token-level probability difference between the two backends constitutes the black-box discrepancy signal $\delta_{i,t}$, which (i) is injected into the GRPO advantage to regularize the policy toward the discrepancy tolerance region, and (ii) drives the dual gradient ascent on the Lagrangian multiplier $\lambda$. The resulting update rule requires only a single extra forward pass per step on the training backend beyond the standard GRPO pipeline, and introduces no additional learnable parameters besides the scalar $\lambda$.

\begin{algorithm}[H]
\caption{\ourmethod: Discrepancy-Constrained Group Relative Policy Optimization.}
\SetAlgoLined
\KwIn{Initial policy $\theta_0$; initial multiplier $\lambda_0$; tolerance budget $c$; dual learning rate $\eta_\lambda$; multiplier range $[\lambda_{\min},\lambda_{\max}]$; clipping range $(\epsilon_{\text{low}},\epsilon_{\text{high}})$; group size $G$; total steps $K$.}
\KwOut{Trained policy parameters $\theta_K$.}
\For{$k = 0, 1, \dots, K-1$}{
  Sample a prompt batch $\{q_b\}_{b=1}^{B}\sim P(Q)$\;
  \textbf{Rollout (inference backend):} for each $q_b$, sample $G$ responses $\{o_{b,i}\}_{i=1}^{G}\!\sim\!\pi_{\theta_{\mathrm{old}}}^{\mathrm{Inf}}(\cdot\mid q_b)$ and record per-token probabilities $\pi_{\theta_{\mathrm{old}}}^{\mathrm{Inf}}(o_{b,i,t}\mid q_b,o_{b,i,<t})$\;
  \textbf{Recompute (training backend):} run a forward pass on the same tokens to obtain $\pi_{\theta_{\mathrm{old}}}^{\mathrm{Train}}(o_{b,i,t}\mid q_b,o_{b,i,<t})$\;
  \textbf{Discrepancy:} $\delta_{b,i,t} \leftarrow \bigl|\pi_{\theta_{\mathrm{old}}}^{\mathrm{Train}}(o_{b,i,t}\mid q_b,o_{b,i,<t}) - \pi_{\theta_{\mathrm{old}}}^{\mathrm{Inf}}(o_{b,i,t}\mid q_b,o_{b,i,<t})\bigr|$\;
  \textbf{Rewards \& GRPO advantage:} compute $r_{b,i}^{\mathrm{task}}$ via the verifiable verifier and
  $\hat A_{b,i}=\dfrac{r_{b,i}^{\mathrm{task}}-\operatorname{mean}(\{r_{b,i}^{\mathrm{task}}\}_{i=1}^{G})}{\operatorname{std}(\{r_{b,i}^{\mathrm{task}}\}_{i=1}^{G})}$\;
  \textbf{Discrepancy-penalized advantage:} $\hat A_{b,i,t}^{\texttt{DCMDP}} \leftarrow \hat A_{b,i} - \lambda_k\cdot \delta_{b,i,t}$\;
  \textbf{Policy update (primal):} perform a GRPO clipped surrogate step on $\theta$ using $\hat A^{\texttt{DCMDP}}$:
  $$\theta_{k+1}\leftarrow \theta_k + \nabla_\theta\!\!\left[\tfrac{1}{B G}\!\!\sum_{b,i}\tfrac{1}{|o_{b,i}|}\!\!\sum_{t}\min\!\left(\gamma_{b,i,t}(\theta)\hat A^{\texttt{DCMDP}}_{b,i,t},\;\operatorname{clip}(\gamma_{b,i,t}(\theta),1{-}\epsilon_{\text{low}},1{+}\epsilon_{\text{high}})\hat A^{\texttt{DCMDP}}_{b,i,t}\right)\right];$$
  \textbf{Multiplier update (dual):} $\;\bar\delta_k \leftarrow \dfrac{\sum_{b,i,t}\delta_{b,i,t}}{\sum_{b,i}|o_{b,i}|}$,\; $\lambda_{k+1}\leftarrow \Pi_{[\lambda_{\min},\lambda_{\max}]}\!\bigl(\lambda_k + \eta_\lambda(\bar\delta_k - c)\bigr)$\;
}
\label{alg:algorithm}
\end{algorithm}

\section{Optimal Policies and Toy Analysis for Alternative Discrepancy Models}
\label{app:optimal_policy_toy}

\subsection{Local Optimal Policies under Three Discrepancy Models}
\label{app:optimal_policy_proof}

At a fixed decoding state $s$, consider the train-side policy $\pi^{\mathrm{Train}}(\cdot\mid s)$ and the deployed inference policy $\pi^{\mathrm{Inf}}(\cdot\mid s)$. We analyze the local surrogate
\begin{equation}
\max_{\pi^{\mathrm{Train}}(\cdot\mid s)\in\Delta}
\left\{
\sum_a \pi^{\mathrm{Train}}(a\mid s)A(s,a)
-\lambda\,\Omega_s\big(\pi^{\mathrm{Train}}(\cdot\mid s),\pi^{\mathrm{Inf}}(\cdot\mid s)\big)
\right\},
\label{eq:app_local_surrogate}
\end{equation}
where $\Delta$ is the action simplex, $A(s,a)$ is the advantage, and $\lambda\ge 0$ is the penalty weight.

To derive closed-form or semi-closed-form optimal policies, we use the squared mismatch as the local constraint surrogate. Compared with the absolute-value form, the quadratic penalty is smooth and differentiable at zero, yields cleaner stationarity conditions, and penalizes larger discrepancies more strongly, making the geometry of the optimum analytically more tractable.

\begin{theorem}
\label{thm:three_penalties_optimal_policy}
Suppose $\pi^{\mathrm{Train}}(\cdot\mid s)$ is an interior maximizer of \eqref{eq:app_local_surrogate}. Under the quadratic discrepancy models
\begin{equation}
\Omega_s^{\mathrm{diff}}
=
\sum_a \pi^{\mathrm{Inf}}(a\mid s)\big(\pi^{\mathrm{Train}}(a\mid s)-\pi^{\mathrm{Inf}}(a\mid s)\big)^2,
\label{eq:app_omega_diff}
\end{equation}
\begin{equation}
\Omega_s^{\mathrm{ratio}}
=
\sum_a \frac{\big(\pi^{\mathrm{Train}}(a\mid s)-\pi^{\mathrm{Inf}}(a\mid s)\big)^2}{\pi^{\mathrm{Inf}}(a\mid s)},
\label{eq:app_omega_ratio}
\end{equation}
and
\begin{equation}
\Omega_s^{\mathrm{log}}
=
\sum_a \pi^{\mathrm{Inf}}(a\mid s)\left(\log\frac{\pi^{\mathrm{Train}}(a\mid s)}{\pi^{\mathrm{Inf}}(a\mid s)}\right)^2,
\label{eq:app_omega_log}
\end{equation}
the corresponding stationary policies satisfy
\begin{equation}
\pi_{\mathrm{diff}}^*(a\mid s)
=
\pi^{\mathrm{Inf}}(a\mid s)+\frac{A(s,a)-\nu(s)}{2\lambda\,\pi^{\mathrm{Inf}}(a\mid s)},
\label{eq:app_pi_diff}
\end{equation}
\begin{equation}
\pi_{\mathrm{ratio}}^*(a\mid s)
=
\pi^{\mathrm{Inf}}(a\mid s)+\frac{A(s,a)-\nu(s)}{2\lambda}\,\pi^{\mathrm{Inf}}(a\mid s),
\label{eq:app_pi_ratio}
\end{equation}
and
\begin{equation}
\pi_{\mathrm{log}}^*(a\mid s)
=
-\pi^{\mathrm{Inf}}(a\mid s)\frac{2\lambda}{A(s,a)-\nu(s)}
W\!\left(-\frac{A(s,a)-\nu(s)}{2\lambda}\right),
\label{eq:app_pi_log}
\end{equation}
where $\nu(s)$ is the normalization multiplier and $W(\cdot)$ is the Lambert-$W$ function. If any interior solution violates non-negativity, the optimum is obtained by the usual KKT projection back to the simplex.
\end{theorem}

\begin{proof}
Introduce the Lagrangian
\begin{equation}
\mathcal{L}_s(\pi,\nu)
=
\sum_a \pi^{\mathrm{Train}}(a\mid s)A(s,a)
-\lambda\,\Omega_s\big(\pi^{\mathrm{Train}}(\cdot\mid s),\pi^{\mathrm{Inf}}(\cdot\mid s)\big)
-\nu(s)\left(\sum_a \pi^{\mathrm{Train}}(a\mid s)-1\right).
\label{eq:app_lagrangian}
\end{equation}
For the probability-difference model \eqref{eq:app_omega_diff},
\[
\frac{\partial \mathcal{L}_s}{\partial \pi^{\mathrm{Train}}(a\mid s)}
=
A(s,a)-2\lambda\,\pi^{\mathrm{Inf}}(a\mid s)\big(\pi^{\mathrm{Train}}(a\mid s)-\pi^{\mathrm{Inf}}(a\mid s)\big)-\nu(s)=0,
\]
which yields \eqref{eq:app_pi_diff}. For the probability-ratio model \eqref{eq:app_omega_ratio},
\[
\frac{\partial \mathcal{L}_s}{\partial \pi^{\mathrm{Train}}(a\mid s)}
=
A(s,a)-2\lambda\,\frac{\pi^{\mathrm{Train}}(a\mid s)-\pi^{\mathrm{Inf}}(a\mid s)}{\pi^{\mathrm{Inf}}(a\mid s)}-\nu(s)=0,
\]
which yields \eqref{eq:app_pi_ratio}. For the log-probability model \eqref{eq:app_omega_log},
\[
\frac{\partial \mathcal{L}_s}{\partial \pi^{\mathrm{Train}}(a\mid s)}
=
A(s,a)-2\lambda\,\pi^{\mathrm{Inf}}(a\mid s)\frac{\log(\pi^{\mathrm{Train}}(a\mid s)/\pi^{\mathrm{Inf}}(a\mid s))}{\pi^{\mathrm{Train}}(a\mid s)}-\nu(s)=0.
\]
Let $c(s,a)\triangleq A(s,a)-\nu(s)$ and $y(a)\triangleq \pi^{\mathrm{Train}}(a\mid s)/\pi^{\mathrm{Inf}}(a\mid s)$. Then the stationarity condition becomes
\[
c(s,a)\,y(a)=2\lambda \log y(a),
\]
or equivalently
\[
y(a)\exp\!\left(-\frac{c(s,a)}{2\lambda}y(a)\right)=1.
\]
Applying the Lambert-$W$ function gives
\[
y(a)
=
-\frac{2\lambda}{c(s,a)}\,W\!\left(-\frac{c(s,a)}{2\lambda}\right),
\]
and substituting back yields \eqref{eq:app_pi_log}. The simplex correction follows from the KKT conditions whenever the interior solution becomes negative.
\end{proof}

The theorem makes the geometry difference explicit. The probability-difference model scales the correction by $1/\pi^{\mathrm{Inf}}(a\mid s)$ and therefore allows more aggressive movement toward low-probability but high-advantage actions. The probability-ratio model scales the correction by $\pi^{\mathrm{Inf}}(a\mid s)$, which is markedly more conservative in the tail. The log-probability model is locally close to the ratio form but remains globally non-linear, providing an intermediate geometry between these two extremes.

\subsection{A One-State Toy Experiment}
\label{app:toy_experiment}

To complement the theorem, we consider a one-state toy problem that isolates how the three discrepancy models reshape the policy under the same advantage signal. The action set is designed to cover six distinct patterns, pairing sampling probability with advantage sign: one \High Prob.\ \Pos Adv.\ action and one \High Prob.\ \Neg Adv.\ action; one \Medium Prob.\ \Pos Adv.\ action and one \Medium Prob.\ \Neg Adv.\ action; and one \Low Prob.\ \Pos Adv.\ action and one \Low Prob.\ \Neg Adv.\ action. Table~\ref{tab:app_toy_patterns} summarizes these patterns.

\begin{table}[t]
\centering
\caption{Toy action patterns for the one-state analysis.}
\label{tab:app_toy_patterns}
\begin{tabular}{lll}
\toprule
Action & Pattern & Setting \\
\midrule
$a_1$ & \High Prob.\ \Neg Adv. & $\pi^{\mathrm{Inf}}(a_1\mid s)=0.30,\; A(s,a_1)=-0.1$ \\
$a_2$ & \High Prob.\ \Pos Adv. & $\pi^{\mathrm{Inf}}(a_2\mid s)=0.30,\; A(s,a_2)=+0.1$ \\
$a_3$ & \Medium Prob.\ \Neg Adv. & $\pi^{\mathrm{Inf}}(a_3\mid s)=0.15,\; A(s,a_3)=-0.5$ \\
$a_4$ & \Medium Prob.\ \Pos Adv. & $\pi^{\mathrm{Inf}}(a_4\mid s)=0.15,\; A(s,a_4)=+0.5$ \\
$a_5$ & \Low Prob.\ \Neg Adv. & $\pi^{\mathrm{Inf}}(a_5\mid s)=0.05,\; A(s,a_5)=-1.0$ \\
$a_6$ & \Low Prob.\ \Pos Adv. & $\pi^{\mathrm{Inf}}(a_6\mid s)=0.05,\; A(s,a_6)=+1.0$ \\
\bottomrule
\end{tabular}
\end{table}

For each discrepancy family, we solve
\begin{equation}
\pi^*(\lambda)
=
\arg\max_{\pi\in\Delta}\left\{\langle \pi,A\rangle-\lambda\,\Omega(\pi,\pi^{\mathrm{Inf}})\right\},
\label{eq:app_toy_opt}
\end{equation}
and choose $\lambda$ so that the corresponding quadratic budget is active, $\Omega(\pi^*(\lambda),\pi^{\mathrm{Inf}})\approx C$. Figures~\ref{fig:app_toy_three_optima_two_side} and \ref{fig:app_toy_three_optima_one_side} visualize the resulting optima under two-sided and one-sided penalties, respectively.

Several patterns are worth emphasizing. First, the direction of reallocation is largely determined by the sign of the advantage: all three penalties consistently suppress the \Neg Adv.\ actions ($a_1$, $a_3$, $a_5$) and promote the \Pos Adv.\ actions ($a_2$, $a_4$, $a_6$). In particular, the paired \Low Prob.\ actions provide a clean control: the \Low Prob.\ \Neg Adv.\ action $a_5$ is suppressed, while the \Low Prob.\ \Pos Adv.\ action $a_6$ is promoted.

The more important difference is the magnitude of this flow. Under probability difference, moving mass onto the strongest \Low Prob.\ \Pos Adv.\ action $a_6$ is comparatively cheap, so the optimizer concentrates probability on $a_6$ most aggressively and tends to produce a sparser solution. Probability ratio spreads the reallocation more gradually across $a_2$, $a_4$, and $a_6$, preserving broader support because tail amplification is much more expensive. Log-probability difference remains intermediate: around the inference policy it behaves similarly to the ratio penalty, but once larger relative drifts become beneficial, its non-linearity allows more movement toward $a_6$ than pure ratio regularization while still avoiding the most aggressive tail-seeking behavior of probability difference.

\begin{takeawaybox}
At the same quadratic budget level, the three discrepancy models agree on the \emph{direction} of advantage-weighted reallocation, moving mass away from \Neg Adv.\ actions and toward \Pos Adv.\ actions, but they differ sharply in the \emph{magnitude} of that reallocation. When high advantage lies on low-probability actions, probability difference makes those moves cheapest, probability ratio makes them most expensive, and log-probability difference provides an intermediate compromise.
\end{takeawaybox}

\begin{figure}[t]
\centering
\includegraphics[width=\linewidth]{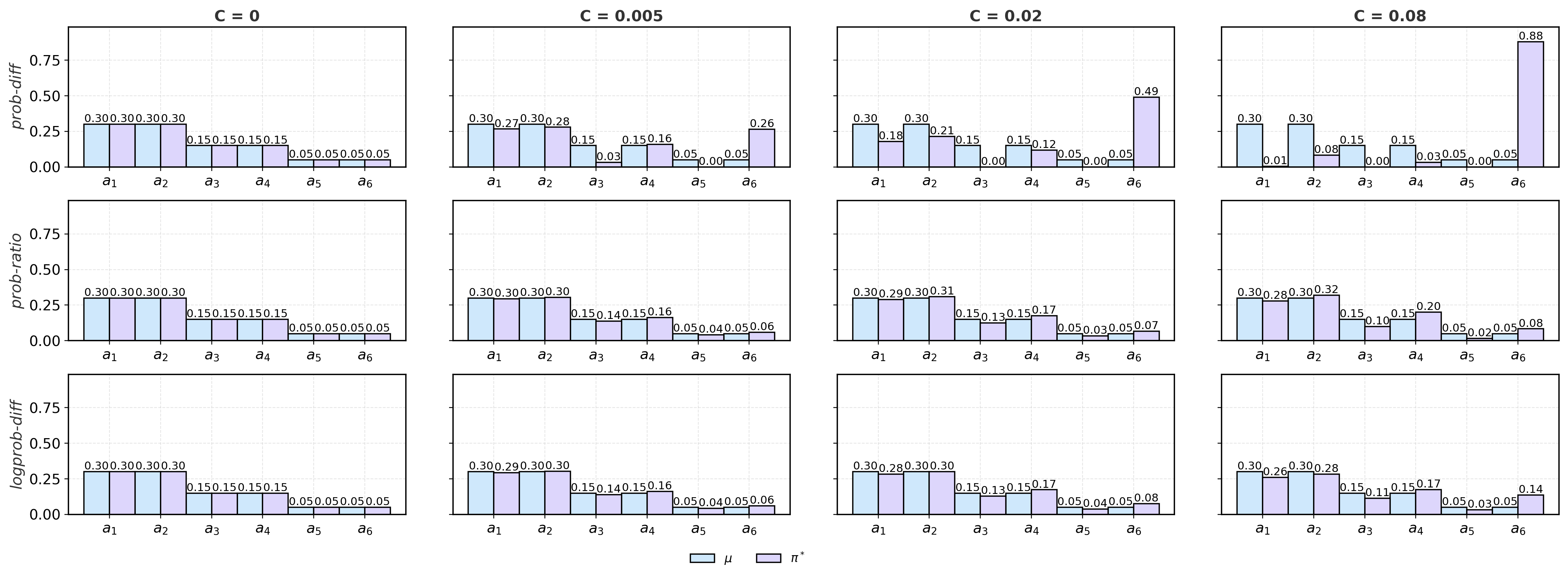}
\caption{Two-sided toy comparison under a shared quadratic mismatch budget. Columns sweep the budget level $C$; rows correspond to probability difference, probability ratio, and log-probability difference penalties.}
\label{fig:app_toy_three_optima_two_side}
\end{figure}

\begin{figure}[t]
\centering
\includegraphics[width=\linewidth]{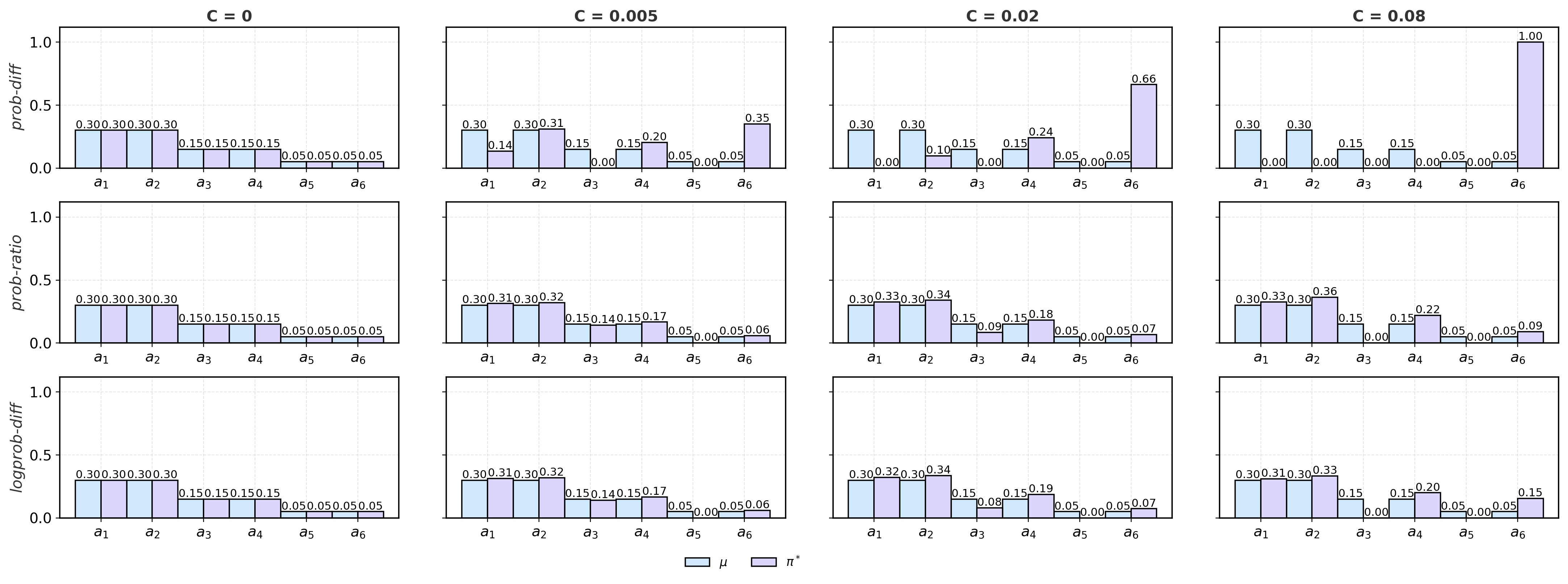}
\caption{One-sided toy comparison under a shared quadratic mismatch budget. Columns sweep the budget level $C$; rows correspond to probability difference, probability ratio, and log-probability difference penalties.}
\label{fig:app_toy_three_optima_one_side}
\end{figure}

This toy analysis highlights the practical message behind the theorem: once the advantage signal is fixed, the discrepancy model acts as a geometric filter on which reward-improving reallocations are considered cheap or expensive. Consequently, the choice among probability difference, probability ratio, and log-probability difference is not merely a matter of metric design; it directly determines how the learned policy redistributes probability mass under advantage-weighted updates.

\section{The Use of Large Language Models}
\label{app:llm_usage}

In this paper, LLMs are only used to polish the writing of some paragraphs to improve clarity and grammar. The key ideas, theoretical analysis, method design, figures, and experimental results are all from the human authors' contributions, and LLMs do not constitute a core component of the proposed methodology.

\section{Licenses}
\label{app:licenses}

The licenses of used assets in this paper are listed as follows:
\begin{itemize}
    \item MATH: MIT License
    \item AMC23: Apache-2.0 License
    \item Minerva Math: No License
    \item OlympiadBench: MIT License
    \item verl: Apache-2.0 License
\end{itemize}

\end{document}